%% file: main.tex
\documentclass{article}
\usepackage[preprint]{neurips_2026}
\usepackage{amsmath,amssymb,amsthm,graphicx,booktabs,url,xcolor,float, enumitem}
\usepackage[hypertexnames=false]{hyperref}

\newtheorem{proposition}{Proposition}

\DeclareMathOperator*{\argmax}{arg\,max}
\title{Do Modules Stay in Their Lane?\\Role Drift in Compound LLM Systems}

\author{%
  \begin{minipage}[t]{0.48\textwidth}
    \centering
    \textbf{Xiaoyang Cao} \\
    \mdseries
    Massachusetts Institute of Technology \\
    \texttt{xycao@mit.edu}
  \end{minipage}%
  \hfill
  \begin{minipage}[t]{0.48\textwidth}
    \centering
    \textbf{Siddarth Srinivasan} \\
    \mdseries
    Harvard University \\
    \texttt{ssrinivasan@seas.harvard.edu}
  \end{minipage} \\[4em] 
  \begin{minipage}[t]{0.48\textwidth}
    \centering
    \textbf{Michiel A. Bakker} \\
    \mdseries
    Massachusetts Institute of Technology \\
    \texttt{bakker@mit.edu}
  \end{minipage}
}

\begin{document}
\maketitle

\begin{abstract}
\input{sections/00_abstract}
\end{abstract}

\input{sections/01_intro}
\input{sections/02_related}
\input{sections/03_theory}
\input{sections/04_experiments}
\input{sections/05_results}
\input{sections/06_conclusion}

\bibliographystyle{plainnat}
\bibliography{references}

\appendix
\renewcommand{\thefigure}{A\arabic{figure}}
\renewcommand{\thetable}{A\arabic{table}}
\setcounter{figure}{0}
\setcounter{table}{0}
\input{sections/90_appendix}


\end{document}

%% file: sections/00_abstract.tex
End-to-end reinforcement learning can improve the accuracy of compound LLM systems, but it does not constrain how modules divide labor internally. We identify \emph{Role Drift}, a failure mode in which modules preserve or improve end-task performance while deviating from their assigned roles through role-violating shortcuts that remain invisible to system-level evaluation. To make role drift observable and controllable, we propose \emph{Role Anchor}, a regularizer that modulates how much each module deviates from its assigned role during end-to-end training. The key idea is to preserve how the role prompt shifts the module's next-token predictions relative to a neutral prompt, which serves as a proxy for the role's intended effect during training. Experiments on two compound LLM pipelines reveal role drift that accuracy alone fails to detect: a decomposer meant to split a question into sub-questions for a separate solver instead plants the answer in them, and a reader meant to answer from retrieved passages instead falls back on parametric memory. In fact, on the decomposer pipeline this shortcut drives most of the apparent RL gain: 86\% of it vanishes once the decomposer is held to its role, indicating that terminal accuracy alone can badly overstate how much a compound system has genuinely learned. Across both pipelines, Role Anchor mitigates role drift at a tunable accuracy cost that varies by pipeline and anchor strength. Additional gradient analysis suggests that the regularizer reduces alignment with the role-drift direction rather than simply suppressing learning.

%% file: sections/01_intro.tex
\section{Introduction}
\label{sec:intro}

Compound LLM systems divide a task among specialized modules: a retriever and a reader, a planner and an executor, a decomposer and a solver. This division of labor is not just an engineering convenience; it encodes the designer's intent about \emph{what each module should do}. The reader should answer from retrieved evidence rather than from memory, so the system stays grounded when the corpus is updated or the question falls outside the model's parametric knowledge. The decomposer should break a problem into sub-questions rather than solve it directly, so that solvers can work in parallel and the reasoning remains auditable. Each module carries an implicit \emph{role utility} that the designer chooses deliberately when partitioning the system.

End-to-end reinforcement learning optimizes these pipelines against a single terminal reward. Terminal accuracy consistently improves, and the community has largely treated this as validation that the system is working as intended~\citep{khattab2024dspy,yuksekgonul2024textgrad,wang2025sysdpo}. But terminal reward measures only whether the final answer is correct; it is agnostic to which module performed which computation. When a module can improve the terminal reward by violating its assigned role, outcome-only RL can find and exploit that shortcut, because the role utility is nowhere in the objective.

\begin{figure*}[t]
\centering
\includegraphics[width=\linewidth]{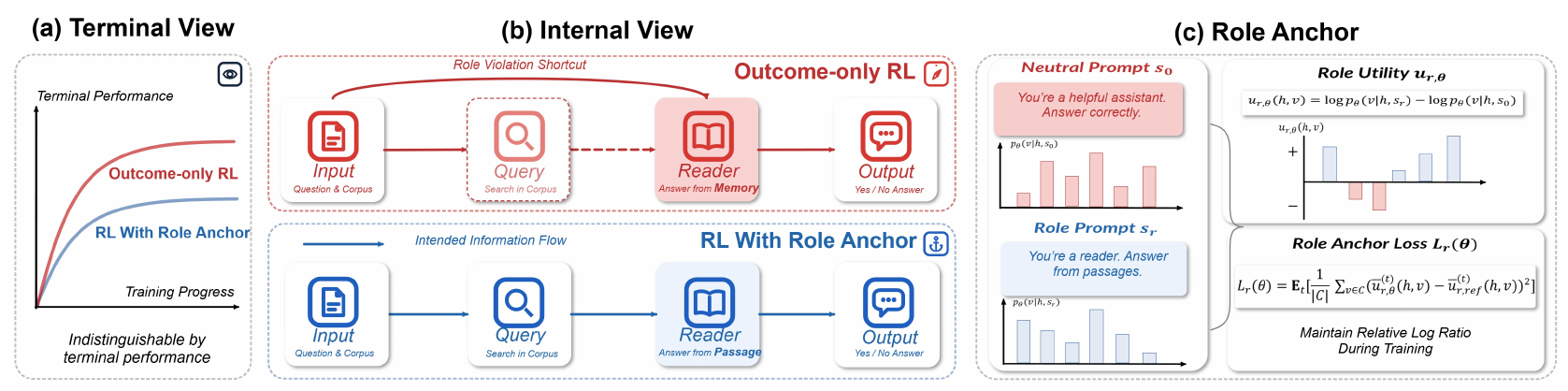}
\caption{\textbf{Same terminal success, different computation.} (a)~Outcome-only RL and RL with Role Anchor are indistinguishable by terminal performance. (b)~Internally, outcome-only RL drifts to a role-violating shortcut: the Reader answers from parametric memory and the retrieval path becomes irrelevant, while Role Anchor preserves the intended information flow. (c)~Role Anchor reads the role utility $u_{r,\theta}$ off the contrast between the module's predictions under the role and neutral prompts (Eq.~\ref{eq:tilt}) and penalizes changes in its centered value relative to the pre-RL reference (Eq.~\ref{eq:anchor}).}
\label{fig:teaser}
\end{figure*}

We call this failure \emph{Role Drift}: a module's effective behavior diverges from its assigned role while the system's terminal accuracy is preserved or improves. We study this behavior through two experiments. First, in a multihop reasoning task, the compound system consists of a Decomposer and a Solver. The Decomposer receives the question and supporting passages, but its role is only to split the problem into abstract sub-questions; the Solver is supposed to answer them. Under outcome-only RL, terminal accuracy rises, but the Decomposer increasingly plants answer entities directly in its sub-questions. Once we constrain it to stay within its role, 86\% of the apparent gain disappears, showing that most of the improvement came from the Decomposer doing the Solver's job. Second, in a retrieval task, a query generator and a frozen retriever provide passages to a Reader, whose role is to answer from the retrieved evidence. Under RL, accuracy again improves, but the Reader becomes less responsive to the passage and falls back on parametric memory instead.

A natural question is why drift should be prevented at all if it improves accuracy. The answer is that deployment requires more than a correct final answer. A drifted Reader is accurate only while its parametric memory matches the corpus; the moment the corpus is updated or a question falls outside what it memorized, the grounding that retrieval exists for is gone. A drifted Decomposer defeats the division of labor itself: once it plants answers in its sub-questions, the work is no longer split, so it cannot be parallelized across solvers, delegated to cheaper ones, or audited step by step. In both cases the system keeps its score while losing the properties it was built for (see Section~\ref{sec:discussion} for a detailed discussion).

We propose \textbf{Role Anchor}, which puts the role back into the objective: it measures how the role prompt shifts a module's predictions relative to a neutral prompt, and penalizes departures of that shift from its pre-RL value. The penalty is a single added term in the policy gradient and leaves the architecture unchanged (Section~\ref{sec:theory}).

\paragraph{Contributions.}
\begin{itemize}[leftmargin=*, topsep=4pt, itemsep=4pt, parsep=0pt]
\item \textbf{Phenomenon:} We identify and name \emph{Role Drift}, a failure mode in which end-to-end RL improves terminal accuracy while modules abandon their assigned roles. We show that much of what looks like learning can in fact be drift, invisible to any terminal-accuracy metric.
\item \textbf{Method:} We propose \textbf{Role Anchor}, a lightweight regularizer that makes role drift both measurable and controllable by anchoring each module's role-prompt effect to its pre-RL behavior.
\item \textbf{Mechanism:} Through gradient analysis, we show that Role Anchor curbs drift by redirecting updates away from the role-violating direction, not by suppressing learning.
\end{itemize}

%% file: sections/02_related.tex
\section{Related Work}
\label{sec:related}

\paragraph{Compound-system optimization.}
Recent compound-AI optimizers treat an LLM pipeline as a program or graph whose prompts, demonstrations, and component policies can be optimized against an end-task metric: DSPy~\citep{khattab2024dspy}, TextGrad~\citep{yuksekgonul2024textgrad}, AFlow~\citep{zhang2024aflow}, Trace/OPTO~\citep{cheng2024trace}, SysDPO~\citep{wang2025sysdpo}, OptiMAS~\citep{wu2025optimas}. These methods supply our optimization context but do not encode module-level role fidelity: a module can change what job it performs if the final metric improves. Role Drift is a failure mode this leaves open.

\paragraph{Underspecification and shortcut learning.}
Underspecification means many models can match validation performance via different internal mechanisms~\citep{damour2020underspecification}. Role Drift shares this spirit at the compound-system level: because terminal reward depends only on the final output, it does not constrain how modules divide labor internally, leaving room for role allocations that differ from the designer's intent. In this sense, Role Drift can be viewed as a module finding a shortcut to improve terminal reward that bypasses its assigned role~\citep{geirhos2020shortcut}, though the shortcut arises from the multi-module structure rather than from spurious data correlations.

\paragraph{Reward hacking.}
Reward hacking~\citep{gao2023scaling,skalse2022defining} arises when a proxy reward has exploitable gaps relative to the true objective. Role Drift is its compositional form in multi-module systems: the terminal reward scores whether the system succeeds, but says nothing about how the modules should divide the work, so a module can drop its assigned role while the final answer stays correct. This is why Role Drift survives even under verifiable rewards (RLVR), which are usually credited with curbing reward hacking~\citep{lambert2024tulu, guo2025deepseek}: the verifier checks the final answer, not which module produced it, so the failure hides inside the pipeline. Since terminal accuracy cannot expose it, we measure Role Drift directly with a role-fidelity probe. The same blind spot is well known in single-model reasoning, where process reward models~\citep{lightman2023lets,uesato2022solving} score intermediate steps instead of only the final answer. Role Anchor brings this to compound systems, deriving its component-level signal from the role-prompt contrast rather than from human-labeled steps.

\paragraph{Inference-time role and faithfulness degradation.}
Even without training, LLMs can fail to maintain assigned behavior.
\citet{liu2023lost} show that models neglect evidence from the middle of long contexts;
\citet{laban2025lost} find a $39\%$ performance drop in multi-turn versus single-turn settings,
with models making premature assumptions and failing to recover;
and \citet{abdulhai2025personas} document persona drift in role-playing dialogue,
where LLMs contradict their assigned persona over successive turns.
Sycophancy~\citep{sharma2023sycophancy} similarly causes models to abandon correct answers
under conversational pressure.
These findings characterize role and faithfulness degradation as inference-time phenomena.
Our work identifies an analogous failure induced by \emph{training}:
outcome-only RL actively pushes modules away from their assigned roles
because the reward is agnostic to internal role allocation.
Role Anchor provides a training-time mitigation
complementary to inference-time interventions.

\paragraph{Context distillation.}
Role Anchor is closest in spirit to context distillation~\citep{snell2022learning,askell2021general}, which transfers a prompted teacher's distribution into model weights offline and drops the prompt at inference. Constitutional AI~\citep{bai2022constitutional} trains with fixed principles via self-critique; prompt internalization~\citep{choi2022prompt,shin2024generative} transfers fixed prompts offline. RLHF and preference objectives regularize against a reference~\citep{ziegler2019fine,ouyang2022training,rafailov2023direct,ethayarajh2024kto} without preserving how the role prompt changes predictions relative to a neutral prompt. The key distinction is what gets preserved: context distillation absorbs the prompt into weights (destroying the role-vs-neutral contrast), while Role Anchor preserves this contrast online during RL (Section~\ref{sec:theory}).

%% file: sections/03_theory.tex
\section{Role Drift and Role Anchor}
\label{sec:theory}

We model a compound LLM system as a directed graph of modules with assigned roles ($V_R$). Each module $i\in V_R$ emits a next-token distribution $p_\theta(v\mid h,s)$, where $v\in\mathcal{V}$ is a candidate next token from the vocabulary $\mathcal{V}$ and $h$ is the generation history; the role prompt $s_r^{(i)}$ specifies the module's intended behavior, and the neutral prompt $s_0^{(i)}$ strips it.

\paragraph{Role utility.}
The effect of a role prompt can be read off the model itself: the same parameters produce different next-token distributions under $s_r$ and $s_0$, and the gap between the two is what the role contributes. We define the \textbf{role utility} of module $i$ as this gap in log space,
\begin{equation}
u_{r,\theta}(h,v) \;=\; \log p_\theta(v\mid h,s_r)\;-\;\log p_\theta(v\mid h,s_0).
\label{eq:tilt}
\end{equation}
The role utility is a score for how the role prompt moves each token: $u_{r,\theta}(h,v)>0$ means the role prompt makes $v$ more likely than it would be under the neutral prompt, and $u_{r,\theta}(h,v)<0$ means less likely. Equivalently, $p_\theta(v\mid h,s_r)=p_\theta(v\mid h,s_0)\,e^{u_{r,\theta}(h,v)}$: the role prompt exponentially reweights the module's default predictions, scaling up the tokens the role favors and scaling down the rest. We call this quantity a \emph{utility} because $p_\theta(\cdot\mid h,s_r)$ is provably the distribution that maximizes the expected $u_{r,\theta}$ under a KL penalty toward the neutral distribution: the module behaves as if it were optimizing this utility while staying close to its default behavior (proof and further properties in Appendix~\ref{app:tilt-derivation}). Reading a utility off the log-ratio of two distributions in this way is standard in reinforcement learning as inference~\citep{levine2018reinforcement,todorov2007linearly}, DPO~\citep{rafailov2023direct}, and maximum-entropy IRL~\citep{ziebart2008maximum}.

\paragraph{Role Drift.}
The role prompt thus equips each module with a utility that steers it toward its intended behavior. We define \textbf{Role Drift} as the failure mode in which $u_{r,\theta}$ shifts during end-to-end optimization: \textit{a module's role utility departs from its pre-training value, so the module no longer performs the job assigned by its role prompt, even though the compound system's terminal performance is preserved or improved.} In our retrieval-augmented reader, this means the reader becomes less sensitive to retrieved evidence; in our decomposer-solver pipeline, the decomposer begins to place answer information in its intermediate questions instead of decomposing the problem. Because terminal reward depends only on the final output, it cannot detect changes in $u_{r,\theta}$: two systems with identical terminal accuracy can have entirely different role utilities. Role Drift is thus an erosion of \emph{role fidelity}, the degree to which a module's behavior conforms to its assigned role, and we measure it with task-specific probes (Section~\ref{sec:results}).

\paragraph{Role Anchor loss.}
\label{sec:method}
To control Role Drift, we anchor the role utility to its pre-RL value. The comparison should ignore one kind of change: adding the same constant to every token's utility leaves the relative preferences, and hence the reweighted distribution after normalization, unchanged, so a uniform shift of $u_{r,\theta}(h,\cdot)$ says nothing about role behavior. We therefore compare mean-centered role utilities over a candidate token set $\mathcal{C}$, taken as the top-$k$ tokens under the current model, which is where the module places most of its probability mass:
\begin{equation}
\bar u_{r,\theta}(h,v) \;=\; u_{r,\theta}(h,v) \;-\; \tfrac{1}{|\mathcal{C}|}\textstyle\sum_{v'\in\mathcal{C}}u_{r,\theta}(h,v').
\label{eq:delta-bar}
\end{equation}
After centering, the sign of $\bar u_{r,\theta}(h,v)$ says whether the role prompt favors $v$ more or less than its average effect across the candidates. We give a formal analysis of why centering is necessary in Appendix~\ref{app:centering}: without it, the anchor would also penalize uniform shifts, which arise whenever RL improves the module's default predictions without touching its role behavior, and the loss would then fight legitimate task learning. We freeze the pre-RL model as a reference and compute $\bar u_{r,\mathrm{ref}}$ from it using the same prompts and the same candidate set $\mathcal{C}_t$. Role Anchor penalizes the squared difference between the current and reference centered role utilities, averaged over token positions $t$:
\begin{equation}
\mathcal{L}_{\mathrm{role}}(\theta) \;=\; \mathbb{E}_{t}\Big[\tfrac{1}{|\mathcal{C}_t|}\textstyle\sum_{v\in\mathcal{C}_t}\big(\bar u^{(t)}_{r,\theta}(v) - \bar u^{(t)}_{r,\mathrm{ref}}(v)\big)^2\Big].
\label{eq:anchor}
\end{equation}
The intuition is that throughout training, the role prompt should keep nudging tokens in the same directions, with similar strength, as it did before RL. The module's absolute predictions are free to change as it gets better at its task; what must be preserved is the extra nudge from swapping the neutral prompt for the role prompt. If that nudge fades, the role prompt has become irrelevant to the module's behavior, which is exactly Role Drift; if it grows far beyond its original magnitude, the role prompt now alters behavior much more than it was designed to, which is rarely what training should produce. The penalty is added to the policy-gradient objective as $\max_\theta \mathbb{E}[R(y)] - \lambda \sum_{i\in V_R} \mathcal{L}_{\mathrm{role}}^{(i)}(\theta)$, where $R(y)$ is the terminal reward and $\lambda$ controls the trade-off between task performance and role preservation.

\begin{proposition}[Role utility preservation]
\label{prop:preservation}
With the role utility defined in Eq.~\ref{eq:tilt}, $\mathcal{L}_{\mathrm{role}}(\theta) = 0$ if and only if the trained model's role utility equals the reference's up to a position-dependent constant: $u_{r,\theta}(h,v) = u_{r,\mathrm{ref}}(h,v) + c(h)$ for all $v, h$. (Proof in Appendix~\ref{app:prop-proof}.)
\end{proposition}

The constant $c(h)$ is the uniform shift that centering discards: absorbed by renormalization, it never changes the module's predictions. Proposition~\ref{prop:preservation} thus says that driving $\mathcal{L}_{\mathrm{role}}$ to zero pins down the role-induced relative preferences while leaving absolute predictions free, so RL can keep improving task accuracy. The premise of Role Anchor is that the role prompt's effect on the pre-RL model is a faithful proxy for the intended role; when this holds, the anchor supervises each module without any task-specific probe design. We verify the premise with an inference-only check before training: the role prompt must produce meaningfully more role-faithful behavior than the neutral prompt on the drift indicator, and should be redesigned otherwise; both benchmarks pass (Appendix~\ref{app:calibration}). Role Anchor is conceptually close to context distillation~\citep{snell2022learning,askell2021general}, which absorbs the prompt's effect into the weights and discards the prompt; it instead preserves the prompt's differential effect ($\bar u_{r,\theta} \approx \bar u_{r,\mathrm{ref}}$) and keeps the prompt alive at inference.

%% file: sections/04_experiments.tex
\section{Experiments}
\label{sec:experiments}

\subsection{Experimental Setup}
\label{sec:setup}

We instantiate Role Drift on two compound pipelines. For each, we describe the task and dataset, the system architecture, and the drift pattern we aim to detect. Probes that quantify each drift pattern are introduced alongside results below.

\paragraph{RAG.}
The first pipeline performs retrieval-augmented question answering: given a yes/no question, the system must retrieve supporting evidence from a Wikipedia corpus and answer from that evidence. It has three components:
\emph{QueryGen} rewrites the user question into a search query,
\emph{Retriever} (a frozen term-matching retriever) retrieves a supporting passage from the corpus, and
\emph{Reader} produces a yes/no answer conditioned on both the question and the retrieved passage.
QueryGen and Reader are both \texttt{Qwen2.5-3B-Instruct}~\citep{qwen2024technical}.
The Reader's assigned role is to answer \emph{from the retrieved evidence} rather than from the question alone or its parametric prior. We train and evaluate on HotpotQA~\citep{yang2018hotpotqa}.

\textit{Drift pattern.}\quad
The Reader becomes less responsive to the retrieved passage:
its answer ceases to change when the supporting passage is swapped to one implying the opposite answer. The drift arises because the Reader has two answer sources, the retrieved passage and its own parametric memory, and the passage, coming from an imperfect upstream retriever, is sometimes wrong. Answering from memory is then often the safer bet, and since the reward checks only the final answer, outcome-only RL steadily reinforces the memory pathway.

\paragraph{Decomposer--Solver (DEC).}
The second pipeline performs multi-hop question answering by explicit decomposition: a complex question is split into simpler sub-questions that are answered in sequence. It has two components:
\emph{Decomposer} receives the question and supporting passages and produces a sequence of abstract sub-questions;
\emph{Solver} answers each sub-question in order and returns the final answer.
Decomposer is \texttt{Qwen2.5-7B-Instruct} and Solver is \texttt{Qwen2.5-0.5B-Instruct}.
We deliberately choose a much weaker Solver: this mirrors practical workflows in which a capable model plans or decomposes while a smaller model executes each sub-task, and it amplifies drift pressure because the Decomposer can easily absorb the Solver's role.
The Decomposer's assigned role is to \emph{decompose} the reasoning while leaving the task-solving entirely to the Solver.
We train and evaluate on MuSiQue-Ans~\citep{trivedi2022musique}.

\textit{Drift pattern.}\quad
The Decomposer encodes answer information directly into the sub-questions (e.g.\ embedding the gold-answer entity), pre-empting the computation the Solver is intended to perform. The leakage arises because the reward scores only the final answer and the Solver is far weaker than the Decomposer: abstract sub-questions often exceed the Solver's ability, while embedding the answer turns it into a copy mechanism that the reward cannot distinguish from genuine solving, so leaking raises reward immediately.

\paragraph{Training.}
All trainable modules use independent LoRA adapters~\citep{hu2022lora} (rank~16).
Both pipelines are optimized with REINFORCE~\citep{williams1992simple} and a group baseline
($k{=}8$ for RAG, $k{=}4$ for DEC;
learning rates $2{\times}10^{-5}$ and $1{\times}10^{-5}$;
10 RL epochs; 3 seeds each).
RAG initializes from a 3-epoch SFT checkpoint;
DEC trains directly from the base model.
The core comparison is \emph{RL alone} (no anchor) vs.\ \emph{RL\,+\,Role Anchor},
with $\lambda$ sweeps over $\{0.02\text{--}0.10\}$ (RAG) and $\{0.05\text{--}1.0\}$ (DEC).
Full hyperparameters and prompts appear in Appendix~\ref{app:hyper}.

%% file: sections/05_results.tex
\subsection{Main Results}
\label{sec:results}

We evaluate both pipelines under outcome-only RL (no anchor) and RL with Role Anchor, using three seeds each. For each pipeline, we track both terminal accuracy and a task-specific role-fidelity probe: evidence-following accuracy for RAG and answer-entity insertion rate for DEC. Figure~\ref{fig:drift-combined} summarizes the results; Table~\ref{tab:cost} reports the numerical values at matched operating points.

\begin{figure}[t]
\centering
\includegraphics[width=\linewidth]{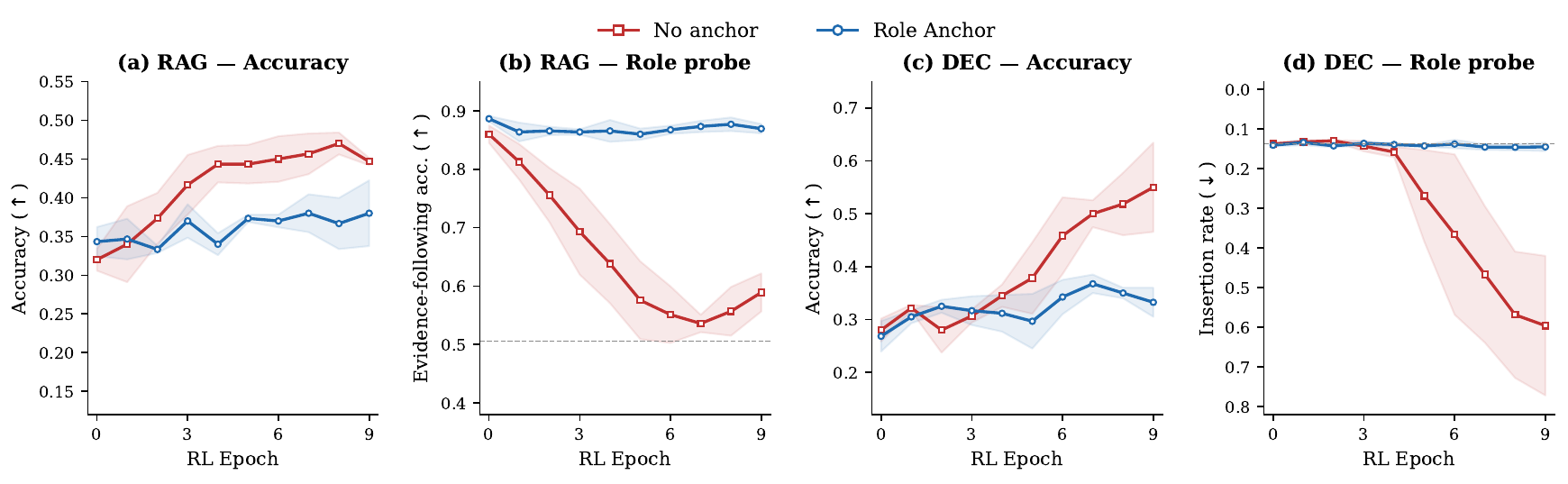}
\caption{\textbf{Outcome-only RL improves accuracy while eroding role fidelity; Role Anchor restores it.}
(a,\,c) Terminal accuracy; (b,\,d) role-fidelity probes. Red: no anchor; blue: Role Anchor.
Without anchor, the RAG Reader stops following retrieved evidence (b) and the DEC Decomposer begins embedding answers into sub-questions (d), even as accuracy rises (a,\,c).
Role Anchor keeps both probes near pre-RL levels. The accuracy gap between arms is the price of preserving the intended division of labor. Arrows mark the better direction; the $y$-axis in (d) is inverted so that higher is better in every panel.}
\label{fig:drift-combined}
\end{figure}

\paragraph{Outcome-only RL induces Role Drift invisible to terminal accuracy.}
\label{sec:drift-exists}

On both pipelines, terminal accuracy improves steadily under outcome-only RL (Figure~\ref{fig:drift-combined}a,\,c). Yet the role-fidelity probes tell a different story.

For RAG, we test whether the Reader still relies on retrieved evidence by swapping the supporting passage to one implying the opposite answer. For instance, for ``Are Mount Tamalpais and Mount Diablo in the same state?'', replacing the passage that places both mountains in California with one placing Mount Diablo in Nevada flips the gold answer from \emph{yes} to \emph{no}.
We call the resulting metric \emph{evidence-following accuracy}: a Reader that genuinely uses the passage should change its answer when the evidence changes (full example in Appendix~\ref{app:rag-example}). Under outcome-only RL, evidence-following accuracy drops from $0.86$ to $0.54$ (Figure~\ref{fig:drift-combined}b), nearly reaching the $0.506$ chance floor of a Reader that ignores passages entirely. The decline is gradual and monotonic, suggesting that the Reader progressively shifts from evidence-based to memory-based answering as RL proceeds.

For DEC, we test whether the Decomposer respects its assigned role by measuring the \emph{answer-entity insertion rate}: the fraction of sub-questions that contain the gold-answer entity, indicating that the Decomposer is solving the task itself rather than delegating to the Solver (concrete example in Appendix~\ref{app:dec-example}). Under outcome-only RL, the insertion rate rises from $0.14$ to $0.60$ (Figure~\ref{fig:drift-combined}d). Unlike RAG's gradual decline, DEC exhibits an abrupt phase transition after epoch~4, where the insertion rate triples in a single epoch. This suggests that once RL discovers the leakage shortcut, it rapidly commits to it. The pattern is consistent across all seeds, with the insertion rate increasing by a factor of two to five.

The two pipelines thus exhibit qualitatively different drift dynamics (gradual erosion on RAG versus abrupt onset on DEC), but the same structural pattern: terminal accuracy improves while role fidelity degrades, and the degradation is not captured by terminal evaluation alone.

\paragraph{Role Anchor recovers role-following behavior.}
\label{sec:results-anchor}

Under Role Anchor (blue in Figure~\ref{fig:drift-combined}), both probes remain flat throughout training: the RAG Reader continues to follow retrieved evidence ($0.87$, Figure~\ref{fig:drift-combined}b), and the DEC Decomposer maintains its pre-RL insertion rate (${\sim}0.14$, Figure~\ref{fig:drift-combined}d). Crucially, the modules still improve at the task (Figure~\ref{fig:drift-combined}a,\,c), indicating that Role Anchor does not freeze learning but channels it through the intended roles.

Two additional probes corroborate the picture (Table~\ref{tab:cost}). A \emph{random-passage probe} on RAG replaces the retrieved passage with an unrelated one: the anchor Reader's accuracy drops to $0.03$ (correctly refusing to answer without evidence), while the no-anchor Reader retains $0.19$, confirming it has learned to answer from parametric memory regardless of what the passage says. A \emph{passage-removal probe} on DEC removes all passages: accuracy collapses to near zero for both arms ($0.05$ and $0.08$), confirming that both systems depend on passages. The difference is \emph{how} they use them: the unanchored Decomposer extracts answers and embeds them in sub-questions; the anchored Decomposer writes abstract sub-questions and lets the Solver reason over the passages.

\begin{table}[t]
\small
\centering
\begin{tabular}{@{}llcc@{}}
\toprule
Pipeline & Metric & No anchor & Role Anchor \\
\midrule
RAG & Accuracy & $0.447$ & $0.380$ \\
    & Evidence-following acc. & $0.589$ & $0.869$ \\
    & Random-passage acc.\ & $0.187$ & $0.030$ \\
\addlinespace[4pt]
DEC & Accuracy & $0.550$ & $0.297$ \\
    & Insertion rate & $0.596$ & $0.143$ \\
    & Acc.\ without passages & $0.050$ & $0.080$ \\
\bottomrule
\end{tabular}
\caption{\textbf{Accuracy and role-fidelity probes} at epoch~9 (3-seed mean).
\emph{Evidence-following acc.}: Reader accuracy after swapping the passage to one implying the opposite answer (higher means the Reader follows evidence).
\emph{Random-passage acc.}: Reader accuracy when the passage is replaced with an unrelated one (lower means the Reader relies on evidence, not memory).
\emph{Insertion rate}: fraction of Decomposer sub-questions containing the gold-answer entity (lower means the Decomposer respects its assigned role).
\emph{Acc.\ without passages}: accuracy when passages are removed (near-zero for both arms suggests accuracy is passage-mediated).}
\label{tab:cost}
\end{table}

\paragraph{The accuracy cost of role preservation.}
\label{sec:rq4}

Preserving role fidelity comes at an accuracy cost that varies markedly between the two pipelines (Table~\ref{tab:cost}). On RAG, the cost is modest: $-0.067$ in accuracy for a $+0.280$ gain in evidence-following. The anchored Reader still improves over the pre-RL model; it simply improves less than the unanchored arm, and the answers it produces are genuinely grounded in retrieved evidence rather than recalled from parametric memory, a property that matters whenever the retrieval corpus is updated or the question falls outside the model's training distribution.

On DEC, the cost is large, but this is itself the key finding. Unanchored RL improves accuracy by $+0.310$ above the base, but under Role Anchor the improvement is only $+0.057$: across seeds, $86\% \pm 19\%$ of the apparent RL gain vanishes under role constraints. This asymmetry between RAG and DEC reflects a difference in how much ``genuine'' learning each pipeline has beyond the drift shortcut. RAG's Reader has a legitimate improvement pathway (learning to extract answers from evidence more effectively), so most of its gain survives the anchor. DEC's Decomposer, paired with a much weaker 0.5B Solver, has little room for genuine improvement and relies heavily on bypassing the Solver. The accuracy cost is not a limitation of Role Anchor but a diagnostic: it quantifies how much of the system's improvement depends on role violation.

\subsection{How Does Role Anchor Mitigate Drift?}
\label{sec:mechanism}

A natural hypothesis is that Role Anchor mitigates drift by shrinking parameter updates.
The gradient geometry reveals a more nuanced picture
that also illuminates why drift affects the two pipelines differently
(Figure~\ref{fig:gradient}).

\begin{figure}[t]
\centering
\includegraphics[width=\linewidth]{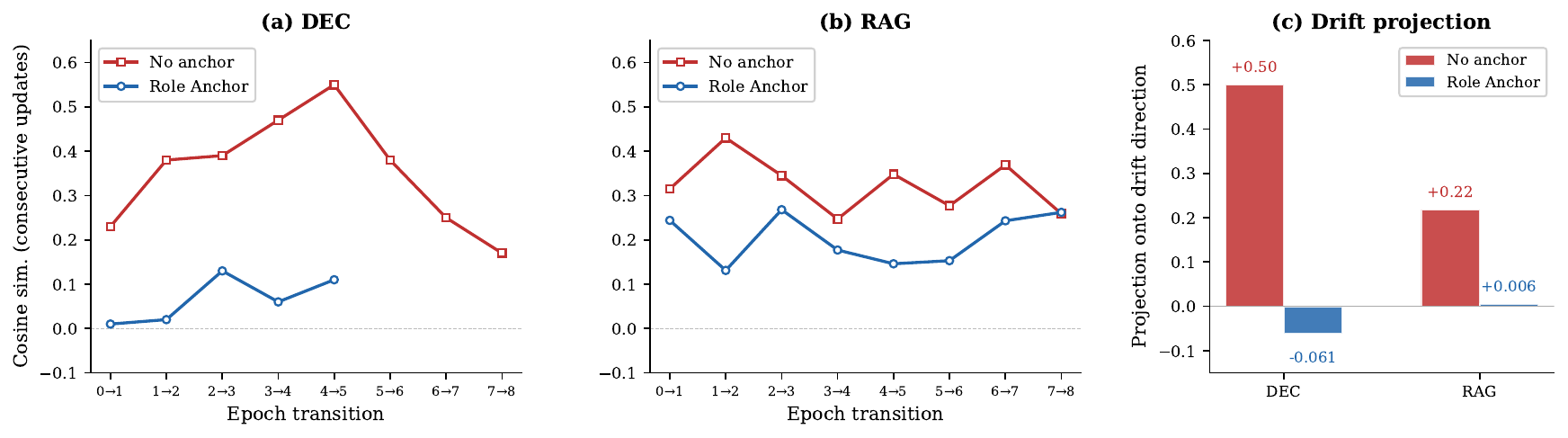}
\caption{\textbf{Role Anchor reduces alignment with the drift direction rather than suppressing learning overall.}
(a,\,b) Cosine similarity between consecutive LoRA updates over training epochs. Without anchor (red), updates cohere along a consistent direction; with anchor (blue), coherence drops sharply on DEC ($0.36 \to 0.07$) but only modestly on RAG ($0.32 \to 0.20$), reflecting the different amounts of legitimate learning available beyond the drift shortcut.
(c) Projection of each method's updates onto the drift direction. Unanchored updates align positively with drift (DEC: $+0.50$, RAG: $+0.22$); anchored updates project to near zero (DEC: $-0.06$, RAG: $+0.006$).}
\label{fig:gradient}
\end{figure}

We extract LoRA adapter weights at each epoch and compute the cosine similarity
between consecutive update vectors $\Delta\theta_t = \theta_t - \theta_{t-1}$.
On DEC (Figure~\ref{fig:gradient}a),
the unanchored Decomposer's updates show a characteristic arc:
coherence rises from $0.23$ to a peak of $0.55$ around epoch~4--5,
then decays as drift saturates.
This peak coincides with the onset of the insertion-rate surge (Figure~\ref{fig:drift-combined}d),
suggesting that directional commitment in parameter space precedes and drives task-level drift.
Under Role Anchor, DEC's cosine drops to $\sim\!0.07$,
nearly indistinguishable from random high-dimensional vectors.
This near-zero coherence is consistent with the main results:
$86\%$ of DEC's unanchored accuracy gain depended on drift (Section~\ref{sec:rq4}),
so once drift is removed, the model has little consistent direction left to follow.

RAG tells a complementary story (Figure~\ref{fig:gradient}b).
Without anchor, updates fluctuate around $0.32$;
with anchor, coherence decreases to $0.20$, a smaller reduction than DEC.
But this is not a weakness of anchor on RAG.
It reflects the asymmetry of Section~\ref{sec:results}: the Reader keeps a legitimate improvement pathway after drift is removed.
This is consistent with RAG's modest accuracy cost ($-0.067$, Table~\ref{tab:cost}):
the anchor Reader keeps learning, just not through the parametric-prior shortcut.

To test whether the anchor specifically blocks the drift direction
rather than reducing learning in general,
we define a \emph{drift direction} as the mean unanchored update vector
during peak drift epochs (ep4--6 for DEC, ep2--6 for RAG)
and project each method's per-epoch updates onto it (Figure~\ref{fig:gradient}c).
Unanchored updates align with the drift direction ($+0.50$ on DEC, $+0.22$ on RAG),
while anchored updates project to approximately zero
($-0.06$ on DEC, $+0.006$ on RAG).
On both pipelines, Role Anchor appears to reduce the alignment of parameter updates with the drift direction,
allowing continued task learning
with less directional commitment to the role-violating shortcut.

\subsection{Ablation Study}
\label{sec:ablation}

We ablate the anchor strength $\lambda$ to understand the accuracy--role-fidelity trade-off (Figure~\ref{fig:lambda-sweep}).

\begin{figure}[t]
\centering
\includegraphics[width=\linewidth]{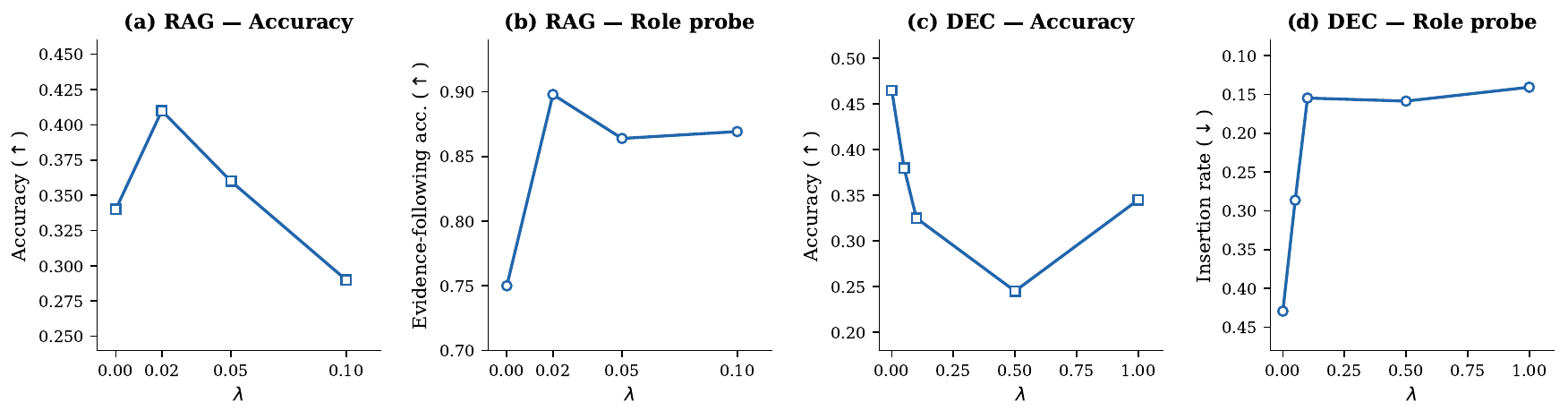}
\caption{\textbf{A small $\lambda$ suffices to suppress most drift.} The anchor strength $\lambda$ provides a continuous trade-off between accuracy (a,\,c) and role fidelity (b,\,d). On RAG, even $\lambda{=}0.02$ closes most of the evidence-following gap. On DEC, insertion rate decreases nearly monotonically with $\lambda$. Arrows mark the better direction; the $y$-axis in (d) is inverted so that higher is better in every panel.}
\label{fig:lambda-sweep}
\end{figure}

One surprising finding is that a small $\lambda$ already captures most of the role-fidelity benefit. On RAG, $\lambda{=}0.02$ raises evidence-following accuracy from $0.75$ to $0.90$ (Figure~\ref{fig:lambda-sweep}b) while also improving task accuracy to $0.41$, the highest value across all tested settings including the unanchored baseline ($0.34$). This suggests that mild role enforcement can help rather than hurt performance when the drift shortcut (parametric memory) is noisier than the intended pathway (passage evidence). Beyond $\lambda{=}0.02$, evidence-following saturates while accuracy declines.

On DEC, the trade-off is qualitatively different: accuracy decreases in the mid-range ($\lambda{=}0.10$--$0.50$) but partially recovers at $\lambda{=}1.00$ ($0.35$), while insertion rate drops nearly monotonically (Figure~\ref{fig:lambda-sweep}d). The recovery at high $\lambda$ suggests that a sufficiently strong anchor forces the Decomposer to discover a legitimate decomposition strategy that intermediate settings do not fully commit to.

The contrast between the two pipelines reinforces the diagnostic value of the $\lambda$ sweep: on RAG, where legitimate learning exists beyond the shortcut, even a small anchor yields a favorable trade-off; on DEC, where most of the gain is drift-dependent, stronger anchoring is needed and the accuracy cost is correspondingly larger.

\subsection{Discussion: Why Prevent Role Drift?}
\label{sec:discussion}
\label{sec:extensions}

Across both pipelines, Role Drift improved terminal accuracy, but it did so by departing from the implicit role utility that the designer chose when partitioning the system. The experiments let us state concretely what that departure costs.

On RAG, the Reader's role is to answer from retrieved evidence. When the passage and parametric memory disagree, which source to trust is a deployment decision, not an epistemic fact, and on the training distribution memory is often just as accurate, which is why outcome-only RL drifts toward it. The retrieval stage exists for the cases where memory fails: the corpus is updated, the fact is stale, the question concerns entities outside pretraining. Under the contract that the corpus, not memory, is authoritative, the evidence-swap probe is accuracy after the world changes: the unanchored Reader's $0.59$ against the anchored $0.87$ (Table~\ref{tab:cost}) is the gap that opens once corpus and memory diverge, and a drifted Reader makes the entire retrieval pipeline functionally irrelevant.

On DEC, the Decomposer's role is to direct the Solver by producing abstract sub-questions, not to solve the task itself. A key motivation for decomposition is efficiency: a powerful Decomposer plans the work while one or more Solvers execute sub-tasks in parallel. When the Decomposer absorbs the Solver's role by embedding answer entities into sub-questions, this parallelization benefit disappears because the Decomposer is doing all the work, and the Solver is reduced to a copy mechanism. Beyond efficiency, auditability suffers: downstream stakeholders cannot verify whether the system genuinely reasoned through the problem or short-circuited it.

In both cases, the accuracy gain from drift is fragile: it depends on a task-specific shortcut that may not transfer to new downstream modules or out-of-distribution queries. The role is a property of the deployment, not an arbitrary constraint. We do not claim Role Anchor improves terminal accuracy in general; rather, it makes Role Drift measurable and tunable, letting practitioners choose a position on the frontier between role fidelity and accuracy based on their deployment requirements.

%% file: sections/06_conclusion.tex
\section{Conclusion}
\label{sec:conclusion}

We identify Role Drift as a failure mode of compound-system RL in which modules deviate from their assigned roles while terminal accuracy continues to improve. On two pipelines with different architectures and drift patterns, task-specific probes reveal role-fidelity degradation that standard accuracy evaluation does not capture. We propose Role Anchor, a regularizer that preserves each module's role-prompt-induced behavior during training. Its anchor strength $\lambda$ provides a continuous trade-off between accuracy and role fidelity, and the accuracy cost under anchoring serves as a diagnostic for how much of the system's improvement depends on role violation.

Gradient analysis suggests that Role Anchor works by reducing the alignment of parameter updates with the drift direction rather than by slowing learning overall. This is consistent with the observation that anchored systems continue to improve at the task through their intended roles. We believe these findings point to a broader principle: when optimizing compound systems end-to-end, it is worth evaluating not only whether the final answer improves, but whether the improvement arrives through the module pathway the system was designed around.

\section{Limitations and Future Work}

Role Anchor, as formulated in this paper, requires log-probability access under both role and neutral prompts, a frozen reference model, and trainable weights for backpropagation. It therefore does not apply to API-only modules, components without a probabilistic output (e.g., a term-matching retriever), or prompt-optimized systems in which the prompt itself is trained. The underlying ideas, however, are not specific to LLMs. Role Drift can arise in any modular system trained end-to-end on a terminal reward: a vision module trained jointly with a downstream controller on task success can drift from estimating object positions to emitting whatever coordinates make that controller succeed. Role Anchor extends conceptually to such systems whenever the module's role is expressed through a conditioning input that can be ablated in place of the prompt, as in goal-conditioned policies or task-conditioned perception; instantiating the anchor for non-LLM modules is a natural direction for future work.

%% file: sections/90_appendix.tex
\section{Theoretical Details}
\label{app:theorem-proof}

\subsection{Properties of the Role Utility}
\label{app:tilt-derivation}

\paragraph{Well-definedness.}
Each module's distributions arise from a softmax, so $p_\theta(v\mid h,s)>0$ for every token and the log-ratio in Eq.~\ref{eq:tilt} is finite. The role utility $u_{r,\theta}$ is therefore defined for all $(h,v)$, with no assumption required.

\paragraph{Uniqueness up to a constant.}
The role utility is identified only up to an additive per-context constant: a function $u$ satisfies $p_\theta(v\mid h,s_r)\propto p_\theta(v\mid h,s_0)\,e^{u(h,v)}$ if and only if $u=u_{r,\theta}+c(h)$ for some constant $c(h)$, since adding $c(h)$ rescales the right-hand side by $e^{c(h)}$, which renormalization absorbs. Only differences $u_{r,\theta}(h,v)-u_{r,\theta}(h,v')$ carry behavioral meaning. This is the same equivalence-class structure under which a reward is recovered from a policy in DPO~\citep{rafailov2023direct} and a reward is identified in maximum-entropy IRL~\citep{ziebart2008maximum}.

\paragraph{Why ``utility''.}
The name has revealed-preference content: the role-prompted distribution is exactly the behavior of an agent that soft-maximizes $u_{r,\theta}$ while staying close to its neutral behavior,
\begin{equation}
p_\theta(\cdot\mid h,s_r)\;=\;\argmax_{q\,\in\,\Delta(\mathcal{V})}\;\;\mathbb{E}_{v\sim q}\big[u_{r,\theta}(h,v)\big]\;-\;\mathrm{KL}\big(q\,\big\|\,p_\theta(\cdot\mid h,s_0)\big).
\label{eq:variational}
\end{equation}
\emph{Proof.} For any $q\in\Delta(\mathcal{V})$,
$\mathbb{E}_q[u_{r,\theta}]-\mathrm{KL}(q\|p_0)
=\mathbb{E}_q\!\big[\log\tfrac{p_r}{p_0}-\log\tfrac{q}{p_0}\big]
=-\mathrm{KL}(q\,\|\,p_r)$,
which is uniquely maximized at $q=p_r$. \hfill$\square$

The role prompt therefore behaves exactly like an objective that the module trades off against a KL budget around its default behavior, in the same sense in which a language model is ``secretly a reward model''~\citep{rafailov2023direct}. We emphasize that Eq.~\ref{eq:variational} is an interpretation of the definition, not an additional assumption: it holds for any pair of shared-support distributions, and nothing downstream depends on it.

\paragraph{What the definition does not imply.}
Eq.~\ref{eq:tilt} defines $u_{r,\theta}$ from the current model, so it holds at every checkpoint by construction. It does not imply that $u_{r,\theta}$ is preserved across training: an RL update changes $p_\theta$ under both prompts and can change $u_{r,\theta}$ substantially while the definition remains valid at each checkpoint. The Role Anchor loss measures precisely this change; it is not zero by definition.

\subsection{Why Mean-Centering}
\label{app:centering}

\paragraph{Invariance.}
The centered utility $\bar u_{r,\theta}$ of Eq.~\ref{eq:delta-bar} is invariant to adding any constant $c(h)$ to $u_{r,\theta}(h,\cdot)$: it depends only on the equivalence class identified above, and is its unique mean-zero representative on $\mathcal{C}$.

\paragraph{What the uncentered loss would add.}
Let $g(h,v)=u_{r,\theta}(h,v)-u_{r,\mathrm{ref}}(h,v)$ and $\bar g(h)=|\mathcal{C}|^{-1}\sum_{v\in\mathcal{C}}g(h,v)$. Then
\begin{equation}
\tfrac{1}{|\mathcal{C}|}\textstyle\sum_{v\in\mathcal{C}}g(h,v)^2
\;=\;\tfrac{1}{|\mathcal{C}|}\textstyle\sum_{v\in\mathcal{C}}\big(g(h,v)-\bar g(h)\big)^2\;+\;\bar g(h)^2,
\end{equation}
so anchoring the raw utilities equals the Role Anchor loss plus the squared mean shift $\bar g(h)^2$. The mean shift carries no information about role behavior: a uniform change of $u_{r,\theta}(h,\cdot)$ across tokens leaves the reweighted distribution unchanged after normalization. Yet it does carry gradients. Suppose RL sharpens the module's neutral distribution from $p_0$ to $p_0'$ while the role prompt applies the same multiplicative reweighting $\rho(v)$: the reference satisfies $u_{r,\mathrm{ref}}(h,v)=\log\rho(v)-\log Z(h)$ with $Z(h)=\sum_v p_0(v\mid h)\,\rho(v)$, the trained model satisfies $u_{r,\theta}(h,v)=\log\rho(v)-\log Z'(h)$ with $Z'(h)=\sum_v p_0'(v\mid h)\,\rho(v)$, and therefore $g(h,\cdot)\equiv\log Z(h)-\log Z'(h)$ is constant in $v$. Role behavior is perfectly preserved and the Role Anchor loss is zero, but the uncentered loss equals $\big(\log\tfrac{Z(h)}{Z'(h)}\big)^2>0$ and its gradient pushes against the change in the neutral distribution, that is, against legitimate task learning. Mean-centering projects out this direction, which is what licenses the claim that only the role-induced relative preferences are anchored while absolute predictions remain free.

\paragraph{Behavior at and away from the reference.}
At $\theta=\theta_{\mathrm{ref}}$, $\bar u_{r,\theta}=\bar u_{r,\mathrm{ref}}$ and the loss is identically zero. If RL improves terminal performance without changing the role-relative preferences, the loss remains near zero; if it changes them, the loss grows in proportion to the change.

\paragraph{Implementation.}
$\mathcal{C}_t$ consists of the top-$k$ tokens under the current model at position $t$, and $\bar u_{r,\mathrm{ref}}$ is computed on the same $\mathcal{C}_t$, so the two centered utilities are compared pointwise on a shared set.

\subsection{Proof of Proposition~\ref{prop:preservation}}
\label{app:prop-proof}

\begin{proof}
Write $m_\theta(h)=|\mathcal{C}|^{-1}\sum_{v'\in\mathcal{C}}u_{r,\theta}(h,v')$, so that $\bar u_{r,\theta}(h,v)=u_{r,\theta}(h,v)-m_\theta(h)$ (Eq.~\ref{eq:delta-bar}).

\textbf{($\Rightarrow$)} If $\mathcal{L}_{\mathrm{role}}(\theta)=0$, then $\bar u_{r,\theta}(h,v)=\bar u_{r,\mathrm{ref}}(h,v)$ for all $v,h$. Expanding,
$u_{r,\theta}(h,v)-u_{r,\mathrm{ref}}(h,v)=m_\theta(h)-m_{\mathrm{ref}}(h)$ for all $v$. The right-hand side depends only on $h$, so setting $c(h)=m_\theta(h)-m_{\mathrm{ref}}(h)$ yields $u_{r,\theta}(h,v)=u_{r,\mathrm{ref}}(h,v)+c(h)$.

\textbf{($\Leftarrow$)} If $u_{r,\theta}(h,v)=u_{r,\mathrm{ref}}(h,v)+c(h)$, then $m_\theta(h)=m_{\mathrm{ref}}(h)+c(h)$, so $u_{r,\theta}(h,v)-m_\theta(h)=u_{r,\mathrm{ref}}(h,v)-m_{\mathrm{ref}}(h)$, giving $\bar u_{r,\theta}=\bar u_{r,\mathrm{ref}}$ and $\mathcal{L}_{\mathrm{role}}=0$.
\end{proof}
We therefore treat the loss as a soft constraint: under finite $\lambda$ and a terminal reward that conflicts with role preservation, the optimum is an interior trade-off, not a zero-loss point.

\section{Pre-Training Calibration}
\label{app:calibration}

Before anchor training, we verify that the role prompt has a measurable effect on the base model by comparing its behavior under role and neutral prompts on the drift indicators defined in Table~\ref{tab:cost}. The pipelines are described in Section~\ref{sec:setup}. Table~\ref{tab:calibration} confirms that the role prompt improves role fidelity on both benchmarks relative to the neutral prompt, validating the reference model for use in Role Anchor.

\begin{table}[htbp]
\small
\centering
\begin{tabular}{lllrrl}
\toprule
Pipeline & Base model & Indicator & Role & Neutral & Margin \\
\midrule
RAG & Qwen2.5-3B-Instruct & Evidence-following acc.\ (higher is better) & 0.659 & 0.580 & $+0.079$ \\
RAG & Qwen2.5-3B-Instruct & Random-passage acc.\ (lower is better) & 0.205 & 0.272 & $-0.067$ \\
DEC & Qwen2.5-7B-Instruct & Insertion rate (lower is better) & 0.130 & 0.210 & $-0.080$ \\
\bottomrule
\end{tabular}
\caption{Pre-training calibration on the base model. The role prompt produces meaningfully more role-faithful behavior than the neutral prompt on all indicators. Role and neutral prompts are recorded verbatim in Appendix~\ref{app:hyper}.}
\label{tab:calibration}
\end{table}

\section{Hyperparameters and Prompts}
\label{app:hyper}

Table~\ref{tab:hyperparams} summarizes the key hyperparameters for both pipelines. Full experimental setup is in Section~\ref{sec:setup}.

\begin{table}[htbp]
\small
\centering
\begin{tabular}{lll}
\toprule
& RAG & DEC \\
\midrule
Base model & Qwen2.5-3B-Instruct & Qwen2.5-7B-Instruct (Decomposer) \\
& & Qwen2.5-0.5B-Instruct (Solver) \\
LoRA rank & 16 & 16 \\
LoRA alpha & 32 & 32 \\
Adapter sharing & Independent & Independent \\
SFT initialization & 3 epochs & None (base model) \\
RL algorithm & REINFORCE + group baseline & REINFORCE + group baseline \\
Rollouts per question ($k$) & 8 & 4 \\
Learning rate & $2 \times 10^{-5}$ & $1 \times 10^{-5}$ \\
RL epochs & 10 & 10 \\
Seeds & 42, 43, 44 & 42, 43, 44 \\
Anchor $\lambda$ (main result) & 0.05 & 1.00 \\
Anchor $\lambda$ (sweep) & \{0, 0.02, 0.05, 0.10\} & \{0, 0.05, 0.10, 0.50, 1.00\} \\
Anchor top-$k$ tokens & 5 & 5 \\
Evaluation set size & 176 (yes/no slice) & 200 (MuSiQue-Ans dev) \\
GPU & 1$\times$ A100 80\,GB & 1$\times$ A100 80\,GB \\
Approx.\ wall-clock per run & ${\sim}2$\,h (no anchor), ${\sim}2.5$\,h (anchor) & ${\sim}4$\,h (no anchor), ${\sim}4.7$\,h (anchor) \\
\bottomrule
\end{tabular}
\caption{Hyperparameters for the two pipelines.}
\label{tab:hyperparams}
\end{table}

\paragraph{Compute budget.}
All experiments were run on single NVIDIA A100 80\,GB GPUs. Each RAG run (10 RL epochs, 3B Reader + 3B QueryGen) takes approximately 2 hours without anchor and 2.5 hours with Role Anchor (${\sim}22\%$ overhead from two extra frozen-reference forward passes). Each DEC run (10 RL epochs, 7B Decomposer + 0.5B Solver) takes approximately 4 hours without anchor and 4.7 hours with Role Anchor (${\sim}18\%$ overhead). With 3 seeds $\times$ 2 conditions (anchor/no-anchor) $\times$ 2 pipelines, plus $\lambda$-sweep runs, the total compute for reported experiments is approximately 120 A100-hours. Preliminary and failed experiments roughly doubled this figure.

\subsection{Verbatim Prompts}

\paragraph{RAG Reader role prompt:}
\begin{verbatim}
You are a careful Reader. Use the retrieved passages to answer the
user's question. If the passages clearly support an answer, give that
answer; if they do not, say so. Do not rely on outside knowledge that
conflicts with the passages.
\end{verbatim}

\paragraph{RAG Reader neutral prompt:}
\begin{verbatim}
Answer the user's question. Format: a concise final answer.
\end{verbatim}

\paragraph{DEC Decomposer role prompt:}
\begin{verbatim}
You are a question decomposer for multi-hop questions. Given a complex
question, produce a numbered list of 2-4 simpler sub-questions whose
answers chain to answer the original. Example: Question: What is the
population of the country where Mount Everest is located? 1. In which
country is Mount Everest located? 2. What is the population of #1?
Compare: GOOD: "What is the population of #1?" BAD: "What is the
population of Nepal?" (do not write the answer entity from the
passages; use #N instead) Return 2-4 numbered sub-questions. Use #N
for chained answers.
\end{verbatim}

\paragraph{DEC Decomposer neutral prompt:}
\begin{verbatim}
You are a question decomposer for multi-hop questions. Given a complex
question, produce a numbered list of 2-4 simpler sub-questions whose
answers chain to answer the original. Use #N notation to reference the
answer of sub-question N (e.g. 'Where is #1 located?'). Output ONLY
the numbered list, no commentary.
\end{verbatim}

\paragraph{Anchor reference model.} We use the post-SFT checkpoint as the frozen anchor reference. In our LoRA setup, the reference is a frozen copy of the adapter loaded alongside the training adapter; no extra model copy is needed.

\section{Qualitative Examples}
\label{app:examples}

\subsection{RAG: Evidence-Following Probe Example}
\label{app:rag-example}

The following example illustrates the counterfactual evidence-swap probe used on the RAG pipeline. By swapping the retrieved passage to one supporting the opposite answer, we test whether the Reader follows the evidence or relies on parametric memory:

\begin{verbatim}
Question: Are Mount Tamalpais and Mount Diablo in the same state?

Original passages (supports "yes"):
  [1] Mount Tamalpais: Located in California.
  [2] Mount Diablo: Located in California.

Counterfactual passages (supports "no"):
  [1] Mount Tamalpais: Located in California.
  [2] Mount Diablo: Located in Nevada.

Role-faithful Reader:  no   (follows counterfactual evidence)
Role-drifted Reader:   yes  (ignores evidence, relies on memory)
\end{verbatim}

A role-faithful Reader changes its answer when the evidence changes; a drifted Reader answers from parametric memory regardless of what the passages say. The evidence-following accuracy reported in Table~\ref{tab:cost} measures how often the Reader's answer tracks the counterfactual evidence.

\subsection{DEC: Answer-Entity Insertion Example}
\label{app:dec-example}

The following example illustrates the answer-entity insertion pattern in the Decomposer-Solver pipeline. The Decomposer receives a multi-hop question and supporting passages, and should produce abstract sub-questions using \texttt{\#N} references:

\begin{verbatim}
Question: What other county does the county where Imperial is located
share a border with?

Decomposer output:
#1 Where is Imperial, Texas located?
#2 Which counties are located near Pecos County, Texas?
#3 Which other county does Pecos County, Texas share a border with?
\end{verbatim}

The gold entities are \emph{Pecos County} and \emph{Crockett County}. Sub-questions \#2 and \#3 contain \emph{Pecos County, Texas}, which the Decomposer resolved from the passages instead of preserving the abstract reference \texttt{\#1}. This is Role Drift in miniature: the Decomposer reads the answer from the passages and embeds it into the sub-questions, bypassing the Solver's intended role.